\documentclass[letterpaper,conference,10pt]{ieeeconf}
\IEEEoverridecommandlockouts
\usepackage{amsmath,graphicx,epsfig,color,amsfonts}

\graphicspath{{./fig_arxiv/}}

\usepackage{verbatim}
\usepackage{amssymb}
\usepackage{mathtools}
\usepackage{commath}
\usepackage{nth}
\usepackage{cite}
\usepackage{cleveref}
\usepackage{caption}
\usepackage{subcaption}
\usepackage{xurl}
\usepackage[vlined,ruled]{algorithm2e}

\def\real{\mathbb{R}}

\newcommand{\col}[1]{\text{col}\left( #1 \right)}
\newcommand{\until}[1]{\{1,\dots, #1\}}

\newcommand{\ceil}[1]{\left\lceil #1 \right\rceil}

\DeclareMathOperator*{\argmax}{arg\,max}

\newcommand\oprocendsymbol{\hbox{$\square$}}
\newcommand\oprocend{\relax\ifmmode\else\unskip\hfill\fi\oprocendsymbol}

\usepackage[dvipsnames]{xcolor}

\renewcommand{\baselinestretch}{1}

\def \bs {\b}
\def \mc {\mathcal}

\newtheorem{theorem}{Theorem}
\newtheorem{proposition}[theorem]{Proposition}

\renewcommand{\theenumi}{(\roman{enumi}}

\newcommand{\rtwo}{\mathbb{R}^2}
\newcommand{\rone}{\mathbb{R}}

\newcommand{\rt}{\rightarrow}

\newcommand{\T}{\top}

\newcommand{\normal}{\mathcal{N}}
\newcommand{\vor}{\mathcal{V}}
\newcommand{\domain}{\mathcal{D}}

\newcommand{\closs}{\mathcal{L}_{\phi}}

\newcommand{\st}{\mid}

\newcommand{\var}{\text{var}}

\renewcommand{\abs}[1]{\left|#1\right|}
\renewcommand{\set}[1]{\left\{#1\right\}}
\renewcommand{\norm}[1]{\left\lVert #1\right\rVert}
\renewcommand{\b}{\boldsymbol}

\usepackage{caption}
\title{Online Estimation and Coverage Control with \\ Heterogeneous Sensing Information
    \thanks{This work has been supported in part by NSF grant IIS-1734272 and ARO grant W911NF-18-1-0325.}
    \author{Andrew McDonald \hspace{0.5in} Lai Wei \hspace{0.5in} Vaibhav Srivastava
        \thanks{A. McDonald is with the Department of Computer Science and Engineering. 
        Michigan State University, East Lansing, MI 48823 USA.
        {\tt\small e-mail: mcdon499@msu.edu}}
        \thanks{L. Wei and V. Srivastava are with the Department of Electrical and Computer Engineering. 
        Michigan State University, East Lansing, MI 48823 USA. {\tt\small e-mail: \{weilai1, vaibhav\}@msu.edu}}
    }
}

\begin{document}

\maketitle


\begin{abstract}\label{sec:abstract}%
Heterogeneous multi-robot sensing systems are able to characterize physical processes more comprehensively than homogeneous systems. Access to multiple modalities of sensory data allow such systems to fuse information between complementary sources and learn richer representations of a phenomenon of interest. Often, these data are correlated but vary in fidelity, i.e., accuracy (bias) and precision (noise). Low-fidelity data may be more plentiful, while high-fidelity data may be more trustworthy. In this paper, we address the problem of multi-robot online estimation and coverage control by combining low- and high-fidelity data to learn and cover a sensory function of interest. We propose two algorithms for this task of heterogeneous learning and coverage---namely Stochastic Sequencing of Multi-fidelity Learning and Coverage (SMLC) and Deterministic Sequencing of Multi-fidelity Learning and Coverage (DMLC)---and prove that they converge asymptotically. In addition, we demonstrate the empirical efficacy of SMLC and DMLC through numerical simulations.
\end{abstract}


\section{Introduction}\label{sec:introduction}

Heterogeneous multi-robot sensing systems---in which agents measure a phenomenon of interest through multiple sensory modalities---often outperform homogeneous systems in which agents are limited to a single sensory modality. By fusing information across modalities, such systems are able to learn a more accurate representation of the phenomenon of interest. For example, a network of drones and unmanned aquatic vehicles monitoring the waters of a lake or ocean may together be able to detect harmful algal blooms (HABs)~\cite{Hallegraeff2003} more accurately than a homogeneous system consisting exclusively of one or the other. Computational analysis of drone imagery combined with \emph{in-situ} measurements of chemical concentrations by unmanned aquatic vehicles offer a more comprehensive view than that from a single data source: the geographic extent of a HAB may be easier to detect from the characteristic large green patches visible in drone imagery, while the precise levels of cyanotoxins may be more accurately measured through direct analysis of water samples.
In spite of this promise, heterogeneous multi-robot sensing systems bring novel challenges which do not emerge in homogeneous systems. The goal of this paper is to address such challenges associated with the tasks of inference and coverage in heterogeneous multi-robot sensing systems.

Agents performing \emph{coverage}~\cite{Cortes2004} aim to distribute themselves over a region according to nonuniform demands for service, placing more agents in areas of high demand and fewer agents in areas of low demand. 
Such demands may be characterized by a sensory function $\phi$ quantifying the importance of a given point, such that $\phi(\b{x}) > \phi(\b{x}')$ encodes the fact that there is a greater demand for sensing at $\b{x}$ than $\b{x}'$. 
%
%
Originally, the coverage problem was studied under the assumption that agents are homogeneous and have perfect knowledge of the sensory function $\phi$. More recently, coverage has been studied in heterogeneous sensing systems~\cite{Hussein2007,Pimenta2008, Kantaros2015, Arslan2016,Pierson2017,Santos2018,Santos2018Constrained}. The class of heterogeneity  considered in these works include differences in sensing, communication and actuation footprints~\cite{Hussein2007,Pimenta2008,Kantaros2015,Arslan2016,Pierson2017}, along with heterogeneity in the types of sensory functions to be covered~\cite{Santos2018,Santos2018Constrained}. These works maintain the assumption that 
$\phi$ is known \emph{a priori}. However, in practice, $\phi$ is often unknown \emph{a priori}. 


A number of works~\cite{Choi2008, Schwager2009, Todescato2017, Schwager2017, Luo2018, Luo2019, Benevento2020, Wei2021} forego the assumption that $\phi$ is known, requiring $\phi$ to be learned from data in the setting of \emph{online estimation and coverage}. 
These works assume agents are homogeneous, equipped with the same sensory, communication and computation capabilities, collecting data of homogeneous modality and quality. Heterogeneity in sensory capabilities among agents is considered in \cite{Sadeghi2019}, but the problem of online estimation and coverage given data at heterogeneous fidelities, i.e., data of varying accuracy (bias) and precision (noise), remains an open area of research.


We consider online estimation and coverage under the assumption that agents collect data pertaining to the same physical process described by the sensory function $\phi$ at multiple \emph{fidelity levels}. High-fidelity data provides an accurate representation of the process, while low-fidelity data provides a noisy, biased representation of the process. 
This scenario often arises in settings where rough estimates are available from
satellite imagery \cite{Campbell2011}, result from sensor scheduling constraints~\cite{wei2020}, or emerge due to energy constraints~\cite{hero2011sensor}.
%
To effectively leverage the multiple fidelity levels at which agents may observe $\phi$, we propose two online multi-fidelity estimation and coverage algorithms---namely Stochastic Sequencing of Multi-fidelity Learning and Coverage (SMLC) and Deterministic Sequencing of Multi-fidelity Learning and Coverage (DMLC)---along with an asymptotic analysis of their convergence and numerical illustrations of their efficacy.


The remainder of the paper is organized as follows. Section \ref{sec:problem}
formulates the online estimation and coverage problem and describes
the approach used to model heterogeneous sensing information.
Section \ref{sec:analysis} then presents SMLC and DMLC, and proves
their asymptotic convergence.
Section \ref{sec:simulation} compares the performance of SMLC and DMLC with
single-fidelity approaches, presenting
numerical results which illustrate the empirical success of 
the multi-fidelity approach. 
Section \ref{sec:conclusion} summarizes the key contributions of
our work and provides concluding remarks.

\section{Preliminaries \& Problem Formulation}\label{sec:problem}


\subsection{Coverage Problem and Lloyd's Algorithm}

Consider the operation of $N$ mobile sensing agents indexed by $i \in \{1, \ldots, N\}$ in a convex, compact 2D environment $\domain \subset \rtwo$, and assume agents may travel to any point in $\domain$. Suppose there exists a sensory function $\phi: \domain \rt \rone_{\geq 0}$ that measures some quantity of interest $\phi(\b{x})$ at each point $\b{x} \in \domain$. 
Intuitively, $\phi$ can be interpreted as the importance of each $\b{x} \in \domain$, such that $\phi(\b{x}) > \phi(\b{x}')$ encodes the fact that there is greater demand for sensing at $\b{x}$ than $\b{x}'$. 
The $N$ agents are distributed across $\domain$ to form a configuration $\b{\eta} = \{\eta_1, \ldots, \eta_N\}$, where $\eta_i$ corresponds to the location of agent $i$. 
Let $P=\{p_1, \ldots, p_N\}$ be an $N$-partition of $\domain$, such that $p_i \cap p_j = \emptyset$ for all $i \ne j$, and $\domain = \cup_{i=1}^N p_i$.
Assigning each partition cell $p_i \in P$ to the $i$-th agent, the \emph{coverage loss} $\closs(\b{\eta}, P)$ corresponding to configuration $\b{\eta}$, $N$-partition $P$ and sensory function $\phi$ is defined by
\begin{align}\label{eq:coverageloss}
    \closs(\b{\eta}, P) &= \sum_{i=1}^N \int_{p_i} \norm{\b{q} - \b{\eta}_i}^2 \phi(\b{q}) d \b{q},
\end{align}
where $\norm{\cdot}$ denotes the standard Euclidean norm in $\rtwo$. The objective is to reach a configuration-partition pair $(\b{\eta}, P)$ that minimizes $\closs(\b{\eta}, P)$, as this will maximize sensory coverage under the assumption that sensor performance degrades with squared distance \cite{Cortes2004}.

For a fixed configuration $\b{\eta}$, the loss function~\eqref{eq:coverageloss} is minimized by the so called Voronoi partition $\vor_{\domain}(\b{\eta}) = \{v_1, \ldots, v_N\}$ with each partition $v_i$ defined by the set of all points closest to agent $i$:
\begin{align}\label{eq:voronoi}
v_i = \set{\b{q}\in \domain \st \norm{\b{q}-\b{\eta}_i} < \lVert{\b{q} - \b{\eta}_j}\rVert \quad \forall \; j \ne i}.
\end{align}
Likewise, for a fixed partition $P$, the loss function~\eqref{eq:coverageloss} is minimized by configuration $\b{c}(P) =  \{ \b{c}_1 (p_1), \ldots, \b{c}_N (p_N) \}$, where each $\b{c}_i(p_i)$ is the \emph{centroid} of $p_i \in P$ defined by
\begin{align}\label{eq:masscentroid}
    \b{c}_i(p_i) = \frac{1}{m_i(p_i)} \int_{p_i} \b{q} \phi(\b{q}) d \b{q}, \;\;  m_i(p_i) = \int_{p_i} \phi(\b{q}) d \b{q}.
\end{align}

A Voronoi partition $\vor_{\domain}(\b{\eta}^*)$ is called a
\emph{centroidal Voronoi partition} \cite{Du1999} if 
$ \b{\eta}^* = \b{c} (\vor_{\domain}(\b{\eta}^*)). $
In general, it is hard to find the global minimum of~\eqref{eq:coverageloss} due to the nonconvexity of $ \closs(\b{\eta}, P) $. However, a configuration corresponding to a centroidal Voronoi partition is considered to be an efficient solution to the coverage problem~\cite{Cortes2004}. To 
achieve such a partition, one
may iteratively apply Lloyd updates \cite{Lloyd1982} initialized from an arbitrary starting configuration $\b{\eta}^{(0)}$ using
\begin{align}\label{eq:lloyd}
   \b{\eta}^{(t+1)} = \set{ \b{c}_1 (\vor_{\domain}(\b{\eta}^{(t)})_1), \ldots, \b{c}_N (\vor_{\domain}(\b{\eta}^{(t)})_N) }. 
\end{align}
It is known that $\vor_{\domain}(\b{\eta}^{(t)})$ converges to a centroidal Voronoi partition as $t\rightarrow \infty$.

\subsection{Modeling Heterogeneous Sensing Data as MFGP}\label{sec:mfgp}


Suppose there are a total of $s$ types of sensory modalities, each of which is indexed by fidelity level $f \in \{1, \ldots, s\}$. The sensory function at each fidelity level $f<s$ denoted by $ \phi_f: \domain \rt \rone_{\geq 0}$ can be viewed as an approximation of $\phi$, and we assume $\phi_s = \phi$ is the ground truth. We model $\phi_1,\ldots,\phi_s$ as a Multi-Fidelity Gaussian Process (MFGP)~\cite{Kennedy2000} in which 
\begin{align}\label{eq:mfgp}
    \phi_f = \rho_{f-1} \phi_{f-1} + \delta_f, \quad \text{for } f \in \{2,\ldots,s\},
\end{align}
where $\rho_{f-1}>0$ is a scaling factor and $\delta_f \sim GP(\mu^f(\b{x}), k^f(\b{x}, \b{x}')), \; f \in \until{s}$ are mutually independent GPs. With this structure, $\delta_f$ refines the information from lower fidelity levels to compose higher fidelities.

At each fidelity $f$, $\phi_f $ can be accessed by taking a sample at $\b{x} \in \domain$ to obtain a measurement of form $y_f=\phi_f(\b{x})+\epsilon_f$, where $\epsilon_f \sim \normal(0, \sigma_f^2)$ is an additive noise. Let the number of measurements at fidelity $f$ be $n_f$ and $X_f \in \real^{2\times n_f}$ be the matrix of sampling locations. Let these $n_f$ measurements be  collected in vector $\b{y}_f \in \real^{n_f}$. We now recall the posterior distribution of $\phi$ from~\cite{wei2020}.

Let $K^i \big(X_f , X_{f'}\big) \in \real^{n_f \times n_f'}$ be a matrix with entries $k^i(\b{x}, \b{x}'), \; \bs x \in \col{X_f}, \; \bs x' \in \col{X_{f'}}$ and $K^i(X_f, \b{x})\in \real^{n_f}$ be a vector with entries $k^i(\b{x'}, \b{x}), \; \bs x' \in \col{X_f}$.  We define $ \rho_{f:f'} = \prod_{i=f}^{f'-1} \rho_i $ if $f<f'$ and $\rho_{f:f'} =1$ if $f=f'$.  Let 
$\b{K}$ be an $s \times s$ block matrix with $\left(f,f'\right)$ submatrix being the covariance matrix of $\b{y}_f$ and $\b{y}_f'$ expressed as
$\b{K}_{f,f'} =   \sum_{i=1}^{\min(f,f')} \rho_{i:f} \rho_{i:f'} K^i\big(X_f , X_{f'}\big)$.
Further, let
$\b{k}(\b{x})$ be a vector constructed by concatenating  $s$ sub-vectors $\b{k}(\b{x}) = \big(\b{k}^{1}(\b{x}), \ldots, \b{k}^{s}(\b{x})\big)$,
where $\b{k}^{f}(\b{x}) = \sum_{i=1}^{f}  \rho_{i:f} \rho_{i:s} K^i(X_{f}, \b{x})$, for $f \in \until{s}$.
We denote the diagonal matrix of sampling noise by $\b{\Theta} = \text{diag} \left( \sigma_1^2 I_{n_1}, \ldots, \sigma_s^2 I_{n_s} \right)$.
Let $\mu^i (X_f) \in \real^{n_f}$ be a vector  with entries $\mu^i (\b{x}), \ \bs x \in \col{X_f}$. Let the mean of $\b{y}_f$ be $ \b{m}_f = \sum_{i=1}^f \rho_{i:f} \mu^i (X_f)$. Construct $\b{\nu}$ by concatenating $s$ subvectors, $\bs \nu = \left(\b{y}_1-\b{m}_1, \ldots, \b{y}_s-\b{m}_s\right) $. Then, the posterior mean and covariance functions of $\phi$ are
\begin{align}
    \mu' (\b{x}) &=  \mu (\b{x}) + \b{k}^{\T}(\b{x}) \left(\b{K} +  \b{\Theta}\right)^{-1} \b{\nu} \label{eq:mfgpmean} \\
    k' \left(\b{x}, \b{x}'\right) &= k \left(\b{x}, \b{x}'\right) - \b{k}^{\T}(\b{x}) \left(\b{K} + \b{\Theta} \right)^{-1} \b{k}(\b{x}'), \label{eq:mfgpvar}
\end{align}
where $\mu(\b{x})$ and $k (\b{x},\b{x}')$ are the prior mean function and covariance function of $\phi$ with expression
\begin{align}\label{eq:mfgppri}
    \mu (\b{x}) = \sum_{i=1}^s \rho_{i:s} \mu^i (\b{x}) , \,\, k (\b{x},\b{x}') = \sum_{i=1}^{s} \rho_{i:s}^2 k^i (\b{x},\b{x}').
\end{align}

\subsection{Online Coverage with Heterogeneous Sensing Data}\label{subsec:mfac}

In the online coverage problem, $\phi$ is unknown \emph{a priori} and must be learned from heterogeneous sensing data. The proposed framework handles two types of heterogeneity: 
\begin{enumerate}
    \item Heterogeneity introduced at the start of execution via information from other modalities (e.g., satellite imagery, drone observations, or human input), and
    \item Heterogeneity introduced over the course of execution due to heterogeneity among agent capabilities and/or sampling strategies.
\end{enumerate}
For simplicity of exposition, we focus on the first class of heterogeneity and restrict our presentation to two fidelity levels, low ($\ell$) and high ($h$). We assume low-fidelity sensing data is obtained prior to deployment through modalities described in (i), and high-fidelity sensing data consists of the real time measurements obtained by robots. The ideas presented extend to the second class of heterogeneity and more than two levels of fidelities. 

We assume the data from all sensing sources are gathered at an information center (e.g., a cloud server) and can be transferred to any functional agent instantaneously, meaning each agent is able to compute the same posterior distribution of $\phi$. 
Over the course of execution, agents must balance learning with coverage, collecting observations to refine their estimate of $\phi$ while staying near the centroid $\b{c}_i$ of their Voronoi cell $v_i$.

\section{Algorithm Design and Analysis}\label{sec:analysis}

We propose and analyze two algorithms for online estimation and coverage in settings with heterogeneous sensing information. 
Each algorithm gradually shifts emphasis from learning to coverage as the estimate converges to the true sensory function $\phi$. 
The algorithms are novel in two ways:
\begin{enumerate}
    \item Both leverage observations collected before execution to construct a prior, then refine this estimate, and
    \item Both utilize the MFGP \cite{Kennedy2000} formulation to efficiently fuse data from multiple fidelities.
\end{enumerate}
We refer to the posterior mean and variance of $\phi$ at iteration $t$ and $\b{x}\in \domain$ by $\hat{\phi}^{(t)}(\b{x})$ and $\var_{\phi}^{(t)}(\b{x})$, respectively. Centroids computed using the mean $\hat{\phi}^{(t)}(\b{x})$ are denoted $\hat{\b{c}}_i^{(t)}$.


\subsection{Stochastic Sequencing of Multi-fidelity Learning \\ and Coverage (SMLC)}

Inspired by the server-based algorithm in \cite{Todescato2017}, the SMLC algorithm is characterized by a stochastic decision process governing whether agent $i$ is tasked with \emph{learning} or \emph{coverage} at iteration $t$. A global multi-fidelity GP estimate of $\phi$ is maintained by collecting low-fidelity observations $(\b{X}_\ell, \b{y}_\ell)$ and high-fidelity observations $(\b{X}_h, \b{y}_h)$ on a central server. 

We define $\vor_{\domain}^{(t)}= (v_1^{(t)}, \ldots, v_n^{(t)})$ to be the Voronoi partition associated with estimated centroids $\hat{\b{c}}_i^{(t-1)}$  from the previous iteration, where $\b{c}^{(0)} =\b{\eta}^{(0)}$. Each iteration $t$ begins with an update to the partition $\vor_{\domain}^{(t)}$ using \eqref{eq:voronoi}. Thereafter, each agent stochastically chooses to execute a \emph{learning} step or a \emph{coverage} step. The probability $p_i^{(t)}$ of agent $i$ executing a \emph{learning} step on iteration $t$ is proportional to the maximum posterior variance $M_{v_i}^{(t)}=\max_{\b{x} \in v_i^{(t)}} \var_{\phi}^{(t)}(\b{x})$ within the Voronoi cell $v_i^{(t)}$, and the probability that the agent executes a \emph{coverage} step is $1-p_i^{(t)}$.

On a \emph{learning} step, agent $i$ collects a noisy high-fidelity sample $y_h=\phi_h(\b{x}_i)+\epsilon_h$ with $\epsilon_h \sim \normal(0, \sigma_h^2)$ at the point $\b{x}_{M_{v_i}}^{(t)} = \argmax_{\b{x} \in v_i^{(t)}} \var_{\phi}^{(t)}(\b{x})$ of maximum posterior variance within $v_i^{(t)}$. On a \emph{coverage} step, agent $i$ proceeds to the estimated centroid $\hat{\b{c}}_i^{(t)}$ of its Voronoi cell. 
Before continuing onto the next iteration, the posterior mean and variance are recomputed by incorporating the samples collected by agents which executed a learning step on this iteration using equations~\eqref{eq:mfgpmean} and \eqref{eq:mfgpvar}. As more samples are collected, the maximum posterior variance $M_{v_i}^{(t)}$ within each cell decreases, driving a gradual shift from learning to coverage as $\hat{\phi}^{(t)} \rightarrow \phi$. 
SMLC is summarized in Algorithm \ref{algo:smlc}.

\begin{algorithm}[ht!]	
	\smallskip
	{\footnotesize 
	
		\SetKwInOut{Input}{  Input}
		\SetKwInOut{Set}{  Set}
		\SetKwInOut{Title}{Algorithm}
		\SetKwInOut{Require}{Require}
		\SetKwInOut{Output}{Output}
		
		
		
		\nl Initialize $\hat{\phi}^{(t)}$ with $(\b{X}_\ell, \b{y}_\ell)$ and $\vor_\domain^{(t)}$ with $\b{\eta}^{(0)}$. Compute $M^{(0)}$.
		
		\For{iteration $t=1,2,\hdots$}{
		
    		
    		\nl Agents update $\hat{\phi}^{(t)}$ and $\var_{\phi}^{(t)}$ with 
    		$(\b{X}_h, \b{y}_h)$ from $t-1$.
    		
    		\nl Construct $\vor_\domain^{(t)}$ from
    		$\hat{\b{c}}_i^{(t-1)}$ and $\hat{\phi}^{(t)}$.
    		Estimate $\hat{\b{c}}_i^{(t)}$.
    		
    		\For{agent $i=1,2,\dots, N$}{
    		
				
    		    \nl Compute $M_{v_i}^{(t)}$. Set $p_i^{(t)} = F(M_{v_i}^{(t)} / M^{(0)})$.

    		    \If{$b_i^{(t)} \sim \text{Bernoulli}(p_i^{(t)}) == 1$}{
    		    
        		    
        		    \nl Agent $i$ samples $y_h$ at $\b{x}_{M_{v_i}}^{(t)}$.
        		    
    		    }
    		    \Else{
    		    
        		    
        		    \nl Agent $i$ drives to estimated centroid $\hat{\b{c}}_i^{(t)}$ of $v_i^{(t)}$.
    		    
    		    }
    		}
		}
	
		\caption{SMLC}
		\label{algo:smlc}
	}
\end{algorithm}

\subsection{Deterministic Sequencing of Multi-fidelity Learning \\ and Coverage (DMLC)}

Inspired by the DSLC algorithm in \cite{Wei2021}, the DMLC algorithm proposed here is characterized by a deterministic sequencing of epochs, where each epoch consists of a \emph{learning} phase and a \emph{coverage} phase. Each epoch $e$ lasts for $n_e$ iterations, and successive epochs increase in length such that uncertainty at the end of these epochs reduces exponentially. As before, we assume agents maintain a global GP estimate of $\phi$, have access to low-fidelity data before execution, and collect high-fidelity data over the course of execution.

Two hyperparameters govern the progression of DMLC, namely $\alpha$ and $\beta$. The hyperparameter $\alpha \in (0, 1)$ represents the factor by which maximum posterior variance $M^{(t)}=\max_{\b{x} \in \domain} \var_{\phi}^{(t)}(\b{x})$ is reduced in epoch $e$, while the hyperparameter $\beta > 1$ represents the factor by which epoch length increases. DMLC begins epoch $e$ with a \emph{learning} phase in which high-fidelity observations are collected to reduce $M^{(t)}$ by a factor of $\alpha$, then proceeds to execute a \emph{coverage} phase in which agents are driven by Lloyd iteration \eqref{eq:lloyd}.

Because the posterior variance $\var_{\phi}^{(t)}(\b{x})$ at each point depends only on the number and location of samples under a GP model \cite{Rasmussen2006}, the entire set of \emph{acquisition points} $\b{X}_a^{(e)}$ which must be sampled to reduce $M^{(t)}$ by a factor of $\alpha$ in epoch $e$ may be computed before any observations are actually collected. DMLC takes advantage of this fact, and determines $\b{X}_a^{(e)}$ at the start of epoch $e$ by \emph{virtually} sampling points of maximum posterior variance $\b{x}_{M_{v_i}}^{(t)} = \argmax_{\b{x} \in v_i} \var_{\phi}(\b{x})$ and recomputing the posterior variance with such points taken into account, fully determining $\b{X}_a^{(e)}$ before any motion occurs. Thereafter, the learning phase assigns the acquisition points $\b{x} \in \b{X}_a^{(e)} \cap v_i$ to be sampled by agent $i$, computes near-optimal Traveling Salesperson Problem (TSP) tours through these points, and dispatches agents on each tour. 
As opposed to na\"ive greedy sampling algorithms which drive to points $\b{x}_{M_{v_i}}^{(t)}$ as data is collected, this TSP-inspired approach is more efficient in terms of travel time and energy expenditure. 
At the end of each learning phase, high-fidelity samples collected in $\b{X}_a^{(e)}$ are used to recompute the posterior mean and variance $\hat{\phi}^{(t)}, \; \var_{\phi}^{(t)}$. 

Next, DMLC enters a \emph{coverage} phase by executing Lloyd's algorithm, sending all agents to the estimated centroid $\hat{\b{c}}_i$ of their respective Voronoi cell. Epoch length $n_e$ is updated with $n_{e+1}=\beta n_e$, and the algorithm repeats. This exponential growth in epoch length leads agents to gradually shift their emphasis from learning to coverage in a manner similar to that of SMLC. 
DMLC is summarized in Algorithm \ref{algo:dmlc}.

\begin{algorithm}[ht!]	
	\smallskip
	{\footnotesize  
		\SetKwInOut{Input}{Input}
		\SetKwInOut{Set}{Set}
		\SetKwInOut{Title}{Algorithm}
		\SetKwInOut{Require}{Require}
		\SetKwInOut{Output}{Output}
		
		
		
		\nl Initialize $\hat{\phi}^{(t)}$ with $(\b{X}_\ell, \b{y}_\ell)$ and $\vor_\domain^{(t)}$ with $\b{\eta}^{(0)}$. Compute $M^{(0)}$.
		
		\For{epoch $e=1,2,\hdots$}{
		
			
			\nl Virtually sample $\b{x}_{M}^{(t)}$ to determine $\b{X}_a^{(e)}$, the set of points which must be sampled to reduce $M^{(t)} \leq \alpha^e M^{(0)}$.
			
			
			\nl Compute TSP tours through each subset $\b{X}_a^{(e)} \cap v_i^{(t)}$.
			
			
			\nl Agents sample $y_h$ on tours then update $\hat{\phi}^{(t)}$ and $\var_{\phi}^{(t)}$.

			
			
			\nl \For{$t_e=1,2,\dots, \ceil{\beta^e n_{e_0}}$}{
    			
                Construct $\vor_\domain^{(t)}$ from $\hat{\b{c}}_i^{(t-1)}$ and $\hat{\phi}^{(t)}$.
        		Estimate $\hat{\b{c}}_i^{(t)}$.
        		
        		\nl Each agent $i$ drives to estimated centroid $\hat{\b{c}}_i^{(t)}$
        		    of $v_i^{(t)}$.
        		}
			
		}
		\caption{DMLC}
		\label{algo:dmlc}
	}
\end{algorithm}

Note that SMLC and DMLC offer distinct strengths and weaknesses. SMLC allows agents to respond at an iteration's notice, but lacks strong guarantees in uncertainty reduction and is inefficient in planning travel. DMLC occasionally renders agents unresponsive during learning phases, but provides reliable uncertainty reduction and more efficient travel planning. Neither algorithm is superior: the choice between SMLC and DMLC depends on the problem setting. 

Both algorithms may be implemented in a distributed manner by leveraging gossip-based strategies \cite{Durham2012}, as proposed in the original works \cite{Todescato2017} and \cite{Wei2021}. In addition, both algorithms can be adapted to use a single-fidelity GP by combining samples $(\b{X}_\ell, \b{y}_\ell)$ and $(\b{X}_h, \b{y}_h)$ into a single dataset $(\b{X}, \b{y})$. These single-fidelity implementations are useful for comparison purposes and serve as a baseline in the simulations presented in Section \ref{sec:simulation}.

\subsection{Analysis of SMLC and DMLC}

A key strength of the original algorithms proposed in \cite{Todescato2017} and \cite{Wei2021} are the convergence guarantees they offer. 
We adapt Proposition 2 of \cite{Todescato2017} to show that the estimate $\hat{\phi}^{(t)}$ will indeed converge in probability to $\phi$ under the SMLC and DMLC algorithms, thereby implying that Proposition 1 in \cite{Todescato2017} also holds for SMLC and DMLC. These results are formally presented in the following propositions. 

Let $\hat{\b{c}}_i^{(t)}$ and $\b{c}_i^{(t)}$ denote the centroid of the Voronoi cell assigned to agent $i$ at iteration $t$, computed using Lloyd's algorithm with the estimated function $\hat{\phi}^{(t)}$ and true function $\phi$, respectively. We recall the following result from \cite{Todescato2017}.

\begin{proposition}[\textbf{Lloyd Equivalence, Proposition 1,~\cite{Todescato2017}}]\label{thm:lloydequivalent}
    Assume $\hat{\phi}^{(t)}$ converges in probability to $\phi$. 
    Choose any $\delta \in (0,1)$, $\epsilon > 0$, and integer $T$. 
    Suppose the centroids $\hat{\b{c}}_i^{(t)}$ and $\b{c}_i^{(t)}$ are updated according to
    the standard Lloyd algorithm \eqref{eq:lloyd}.
    Then there exists an iteration $t_0$ such that
    \begin{align}
        P \left( \norm{\hat{\b{c}}_i^{(t_0 + t)} - \b{c}_i^{(t_0 + t)}} \leq \epsilon \right) \geq 1 - \delta
    \end{align}
    holds for any $t \in \{ 0, \ldots, T \}$ and $i \in \{ 0, \ldots, N \}$ if $\hat{\b{c}}_i^{(t_0)}=\b{c}_i^{(t_0)}$. Intuitively, this means the evolution of $\hat{\b{c}}_i^{(t)}$ and $\b{c}_i^{(t)}$ will remain close with high probability for an arbitrary number of iterations.
\end{proposition}

Now, we show both SMLC and DMLC guarantee that $\hat{\phi}^{(t)}$ converges in probability to $\phi$, so that Proposition \ref{thm:lloydequivalent} holds. Let $M^{(t)}=\max_{\b{x} \in \domain} \var_{\phi}^{(t)}(\b{x})$ denote the maximum posterior variance $\var_{\phi}^{(t)}$ at iteration $t$.
Let $\b{x}_M^{(t)} = \argmax_{\b{x} \in \domain} \var_{\phi}^{(t)}(\b{x})$ be the point where $M^{(t)}$ is achieved.

\begin{proposition}[\textbf{Convergence of SMLC Estimate}]\label{thm:smlcconvergence}
    Consider the SMLC algorithm presented in Algorithm \ref{algo:smlc} with the
    multi-fidelity GP model \eqref{eq:mfgp}.
    Let $F: [0, 1] \rt [0, 1]$
    be a continuous and strictly increasing function of $M^{(t)}$.
    Then $\hat{\phi}^{(t)}$ converges in probability to the true function 
    $\phi$.
\end{proposition}
\begin{proof}
    The structure of the SMLC algorithm is identical to that
    of the server-based (SB) algorithm proposed in \cite{Todescato2017},
    differing only in its use of a MFGP.
    Yet, $\hat{\phi}^{(t)}$ is itself a GP with composite covariance function $k(\b{x}, \b{x})$ as shown in~\eqref{eq:mfgppri}. Observations from lower fidelity levels makes $\hat{\phi}^{(t)}$ converge faster to $\phi$. The result follows by Proposition 2 in \cite{Todescato2017}.
\end{proof}

\begin{proposition}[\textbf{Convergence of DMLC Estimate}]\label{thm:dmlcconvergence}
    Consider the DMLC algorithm presented in Algorithm \ref{algo:dmlc}
    with the multi-fidelity GP model \eqref{eq:mfgp}.
    Let $\alpha \in (0,1)$ denote the factor by which
    maximum posterior variance $M^{(t)}$ is reduced in each epoch.
    Then $\hat{\phi}^{(t)}$ converges in probability to the true function.
\end{proposition}
\begin{proof}
    Let $\Bar{\phi}(\b{x}) \sim GP(\Bar{\mu}(\b{x}), \Bar{k}(\b{x},\b{x}'))$ be a GP defined on $\Bar{\domain} \subseteq \domain$, where covariance function $\Bar{k}(\b{x},\b{x}')$ is continuous and $\Bar{k}(\b{x}, \b{x}) \leq \lambda, \forall \b{x} \in \Bar{\domain}$. Let $\sigma^2$ be the variance of additive Gaussian sampling noise. Now, we show the greedy sampling strategy reduces the maximum variance below threshold $\epsilon$ in finite iterations for any fixed $\epsilon$. Note that for every $\epsilon>0$, there exist $ \gamma$, $\nu$ and $n$ that satisfy
    $\lambda - (\lambda - \gamma)^2 / \big(\lambda + \nu + \frac{\sigma^2}{n}\big) \leq \epsilon$.
    Since $\Bar{k}(\b{x},\b{x}')$ is continuous, there exists a
    finite partition $P=\set{p_1, \dots, p_{m}}$ of $\Bar{\domain}$, with $m$ large enough, such that for any $p_i \in P$ and for any pair of points $\b{x}, \b{x}' \in p_i$,
    both $\abs{\Bar{k}(\b{x}, \b{x}) - \Bar{k}(\b{x}, \b{x}')} \leq \gamma$ and $\abs{\Bar{k}(\b{x}, \b{x}) - \Bar{k}(\b{x'}, \b{x}')} \leq \nu$.
    After $nm+1$ sampling rounds, there exists at least one cell $p_i \in P$ that is selected $n_i \ge (n+1)$ times. Let $X \in \real^{2\times n}$ be the matrix of first $n$ sampling locations within $p_i$ and $\b{y} \in \real^{n}$ be the corresponding sampling results. Then, for any $\bs{x}\in p_i$, $\max_{\b{x}' \in p_i } \Bar{k}(\b{x}',\b{x}') \leq \Bar{k}(\b{x},\b{x}) + \nu$. Let $\Bar{K}(X,X) \in \real^{n\times n}$ be the covariance matrix with entries $\{\Bar{k}(\b{x}',\b{x}'')\}_{\b{x}', \b{x}'' \in X}$ and $\Bar{K}(X,\b{x})=\Bar{K}^{\T}(\b{x},X) \in \real^{n\times n}$ be the vector with entries $\{\Bar{k}(\b{x},\b{x}')\}_{\b{x}' \in X}$. Using the Gershgorin circle theorem~\cite{horn2012matrix}, the spectral radius of symmetric matrix $\Bar{K}(X,X) + \sigma^2 I$ is no greater than $n(\Bar{k} (\b{x}, \b{x})+\nu) + \sigma^2$. Thus, for every $\b{x} \in p_i$,
    \begin{align*}
        &\var_{\Bar{\phi}}(\b{x} | X,\b{y}) \\
        &= \Bar{k} (\b{x}, \b{x}) - \Bar{K}(\b{x},X)  \big(\Bar{K}(X,X) + \sigma^2 I \big)^{-1} \Bar{K}(X,\b{x}) \\
        &\leq \Bar{k} (\b{x}, \b{x}) - \frac{\| \Bar{K}(X,\b{x}) \|_2^2}{n (\Bar{k} (\b{x}, \b{x})+\nu) + \sigma^2}\\
        &\leq \Bar{k} (\b{x}, \b{x}) - \frac{(\Bar{k} (\b{x}, \b{x}) - \gamma)^2}{\Bar{k} (\b{x}, \b{x}) + \nu + \frac{\sigma^2}{n}} \leq \lambda - \frac{(\lambda - \gamma)^2}{\lambda + \nu + \frac{\sigma^2}{n}} \leq \epsilon,
    \end{align*}
    where we use the fact $x- {(x - \gamma)^2}/{(x + \nu + \frac{\sigma^2}{n})}$ increases with $x$ for $x>0$. Since the variance at every point in $p_i$ is smaller than $\epsilon$ after $n$ samples and $p_i$ is selected $n_i \ge (n+1)$ times, the variance at every point $\b{x}\in \mc D$ must be smaller than $\epsilon$ when $p_i$ is selected $(n+1)$-th time  by the greedy policy, i.e., after $nm+1$ overall samples have been collected. 
    
    This shows each learning step in DMLC can be finished in finite iterations. Thus, for and $\epsilon$ and $\delta$, there is a finite $t'$ for which $M^{(t)} \leq \delta\epsilon^2$ holds for all $t > t'$. Now, by the Chebyshev inequality, we have
    $P ( | \hat{\phi}^{(t)}(\b{x}) - \Bar{\phi}(\b{x}) |  \geq \epsilon ) \leq \var_{\phi}^{(t)}(\b{x}) / \epsilon^2 \leq \delta$
    at any point $\b{x} \in \domain$ for all $t > t'$, which concludes the proof.
\end{proof}

\section{Simulations}\label{sec:simulation}

Numerical simulations showed that multi-fidelity implementations of SMLC and DMLC outperformed their single-fidelity counterparts.
%
%
In particular, multi-fidelity implementations of SMLC and DMLC incurred lower \emph{cumulative coverage regret} $R_{\phi}$, defined as the sum:
\begin{gather}\label{eq:regret}
    R_\phi = \sum_t r_{\phi}^{(t)}, \quad \text{where}\\
    r_{\phi}^{(t)} = \closs \left(\b{\eta}^{(t)}, \vor_{\domain}(\b{\eta}^{(t)})\right) - \closs\left(\b{c} \big(\vor_{\domain}(\b{\eta}^{(t)})\big), \vor_{\domain}(\b{\eta}^{(t)})\right), \nonumber
\end{gather}
such that $r_{\phi}^{(t)}$ measures the instantaneous deviation of $\b{\eta}^{(t)}$ from the set of centroidal Voronoi partitions generated by $\b{\eta}^{(t)}$.
This metric is nonnegative and achieves 0 at a centroidal Voronoi partition. Unlike the loss $\closs$ defined in \eqref{eq:coverageloss}, it does not depend on the initial configuration $\b{\eta}^{(0)}$, allowing a fair comparison between runs. In addition, it rewards a balance between learning and coverage---agents must accurately estimate $\phi$ in order to accurately estimate $c^{(t)}$.

We simulated each algorithm over a 21 $\times$ 21 point discretization of the unit square $\domain = [0, 1]^2$ with four agents. 
We defined $\phi_h$ and $\phi_\ell$ as the Gaussian functions
\begin{align}
    \phi_f &= v_f^2 \exp \left( 
    - \norm{ \b{x} - \b{c}_f }^2 / 2l_f^2
    \right), \quad f \in \set{h, \ell},
    \label{eq:simphi}
\end{align}
with $v_\ell = \sqrt{5}, \; l_\ell = \sqrt{0.2}, \; \b{c}_\ell = (0.5, 0.5)$ and $v_h = \sqrt{10}, \; l_h = \sqrt{0.05}, \; \b{c}_h = (0.75, 0.75)$ as shown in Fig.~\ref{fig:function}.
%
%
Importantly, $\phi_\ell$ and $\phi_h$ are correlated such that access to low-fidelity samples 
provide insight into the behavior of $\phi_h$, despite the fact that $\phi_\ell$ is biased from the ground truth $\phi_h$. 

Twenty-five low-fidelity samples at the points
$\b{X}_\ell = \{ (0.25i, 0.25j) \in \domain \st 0 \leq i, j \leq 4 \}$
drawn from $\phi_\ell$ with Gaussian noise $\sigma_\ell^2 = 1$
were provided to initialize the learning model of each algorithm.
Thereafter, all samples $(\b{X}_h, \b{y}_h)$ collected by agents
were drawn from $\phi_h$ with Gaussian noise $\sigma_h^2 = 1$.
Because the single-fidelity implementations of SMLC and DMLC
cannot distinguish fidelity levels, they combine all observations
$(\b{X}_\ell, \b{y}_\ell)$ and $(\b{X}_h, \b{y}_h)$
into a single dataset $(\b{X}, \b{y})$ throughout the learning process.
In contrast, the multi-fidelity implementations of SMLC and
DMLC maintain separate datasets.

We simulated the classic coverage control algorithm presented in \cite{Cortes2004} (which assumes perfect knowledge of $\phi_h$) as a baseline to compare against. Single and multi-fidelity GP models were implemented in Python by adapting existing open-source code \cite{PerdikarisCode}, and GP hyperparameters were fit to maximum likelihood estimates given 100 training samples from \eqref{eq:simphi} using standard optimization techniques. 
Each algorithm was simulated with 4 agents for 60 iterations, corresponding to epoch lengths of 4, 8, 16 and 32 in DMLC chosen by setting uncertainty reduction parameter $\alpha=1/\sqrt{2}$ and epoch length multiplier $\beta=2$. Simulations of each algorithm were repeated 50 times with randomly generated initial configurations $\b{\eta}^{(0)}$, and results were averaged. Code and videos are available via GitHub \cite{McDonaldCode}.

\begin{figure}[t!]
    \centering
    \includegraphics[width=0.4\textwidth]{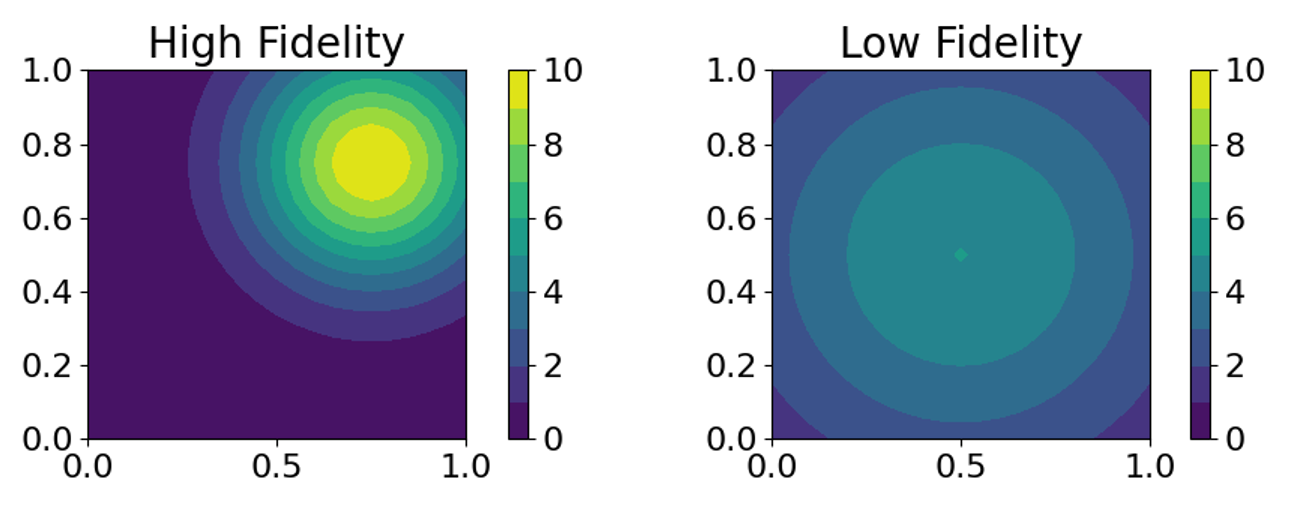}
    \caption{$\phi_h$ (left) and $\phi_\ell$ (right)
    used in simulations \eqref{eq:simphi}.
    }
    \label{fig:function}
    \vspace{-2em}
\end{figure}



\begin{figure*}[hbt!]
    \centering
    \includegraphics[width=0.95\textwidth]{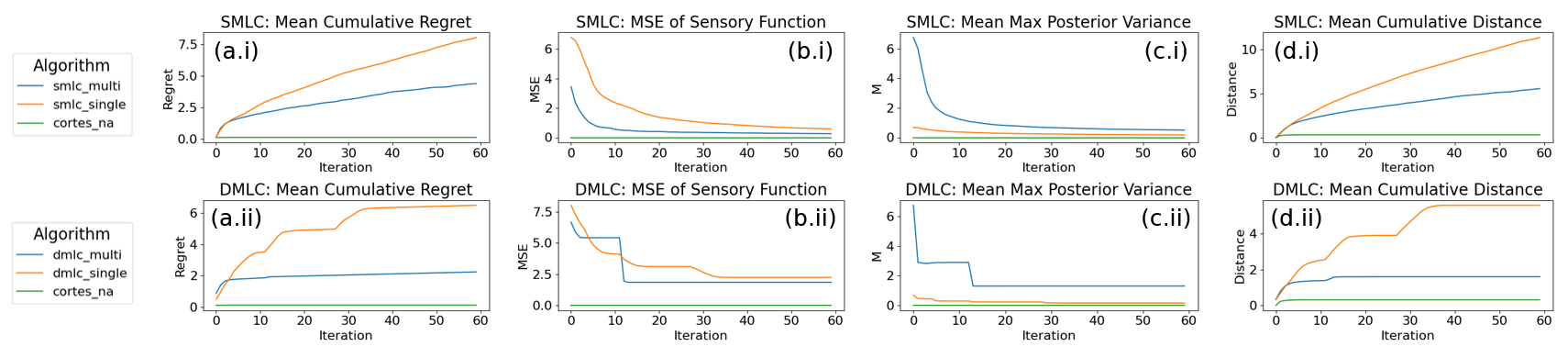}
    \caption{Simulation results averaged over 50 runs per algorithm-fidelity
    pair. Blue lines = multi-fidelity implementations,
    orange lines = single-fidelity implementations, and
    green lines = baseline from \cite{Cortes2004} which assumes perfect knowledge of $\phi_h$. See text for further details.
    }
    \label{fig:performance}
    \vspace{-1em}
\end{figure*}

Key results of our simulations are shown in Fig. \ref{fig:performance}. Panel (a) compares the cumulative coverage regret \eqref{eq:regret} incurred by SMLC and DMLC using single- and multi-fidelity learning models, and shows that the multi-fidelity implementation of both algorithms outperforms the single-fidelity implementation when provided low-fidelity data $(\b{X}_\ell, \b{y}_\ell)$ prior to execution. This is due to the fact the multi-fidelity model is able to structurally distinguish between low-fidelity and high-fidelity data, and estimation of model hyperparameters ensures that (in)consistency between these data are appropriately captured and incorporated. The single-fidelity model does not provide this flexibility. 
%
%
Panel (b) compares the mean squared error $|| \hat{\phi}(\b{X}^*) - \phi(\b{X}^*) ||^2 \; / \; | \b{X}^* |$ of the estimate $\hat{\phi}^{(t)}$, affirming that multi-fidelity methods learn a more accurate representation of $\phi$ when provided low-fidelity data. Panel (c) compares the evolution of maximum posterior variance $M^{(t)}=\max_{\b{x} \in \domain} \var_{{\phi}}^{(t)}(\b{x})$, and highlights another key advantage offered by the multi-fidelity approach: more nuanced quantification of uncertainty. Because the single-fidelity implementations cannot distinguish between observations $(\b{X}_\ell, \b{y}_\ell)$ and $(\b{X}_h, \b{y}_h)$, the mean estimate $\hat{\phi}^{(t)}$ becomes overconfident and does not account for potential bias introduced by low-fidelity observations.
Finally, panel (d) compares the mean distance travelled by each agent over the course of execution. 
Evidently, single-fidelity implementations involve more travel as a result of overconfidence in the mean estimate $\hat{\phi}^{(t)}$: more samples are required to reduce variance by a further multiplicative factor.
Panel (d) also confirms that agents executing DMLC travel less distance on average than agents executing SMLC, as discussed in Section \ref{sec:analysis}.


\section{Conclusion}\label{sec:conclusion}

Heterogeneous multi-robot sensing systems improve upon homogeneous systems by fusing information across sensory modalities and learning more accurate representations of physical processes. However, the method used to combine heterogeneous data plays a crucial role in system performance and merits principled consideration. In this work, we used multi-fidelity Gaussian Processes (GPs) to combine sensory data of different fidelities collected by a heterogeneous multi-robot sensing system. We applied this approach to the task of online estimation and coverage in which an unknown sensory function $\phi$ must be learned and covered by agents, and illustrated scenarios in which single-fidelity approaches fall short. We proposed SMLC and DMLC, two novel adaptive coverage algorithms which leverage the multi-fidelity GP framework,
and proved the asymptotic convergence of each. Finally, we demonstrated the empirical efficacy of SMLC and DMLC through numerical simulations.


\bibliographystyle{IEEEtran}
\bibliography{IEEEabrv, main_arxiv}

\end{document}